\documentclass[preprint,12pt]{elsarticle}

\usepackage{amssymb}
\usepackage{amsmath,amsfonts}
\usepackage{array}
\usepackage[caption=false,font=normalsize,labelfont=sf,textfont=sf]{subfig}
\usepackage{textcomp}
\usepackage{stfloats}
\usepackage{url}
\usepackage{verbatim}
\usepackage{graphicx}
\usepackage{amsthm}
\usepackage{floatrow}
\usepackage{wrapfig,lipsum,booktabs}

\newtheorem{definition}{Definition}
\newtheorem{corollary}{Corollary}
\newtheorem{observation}{Observation}
\usepackage{algpseudocode}
\usepackage{algorithm}
\usepackage{multirow}
\usepackage{mwe}
\usepackage{cancel}
\usepackage{comment}
\usepackage{xcolor}
\usepackage{colortbl}
\newcommand{\blockcomment}[1]{}
\usepackage{pifont}
\newcommand{\cmark}{\ding{51}}%
\newcommand{\xmark}{-} 

\usepackage{float}
\usepackage{paralist, tabularx}

\floatstyle{plaintop}
\restylefloat{table}

\def\BibTeX{{\rm B\kern-.05em{\sc i\kern-.025em b}\kern-.08em
    T\kern-.1667em\lower.7ex\hbox{E}\kern-.125emX}}
\usepackage{balance}

\usepackage{lineno}

\journal{Neurocomputing}

\begin{document}
\begin{frontmatter}



\title{Generalized Mask-aware IoU for Anchor Assignment for Real-time
Instance Segmentation}
 
 \affiliation[tu_darmstadt]{organization={Dept. of Computer Science, TU Darmstadt},
             addressline={Hochschulstraße 10},
             city={Darmstadt},
             postcode={64289},
             state={Hesse},
             country={Germany}}

 \affiliation[five_ai]{organization={Five AI Ltd.},
             addressline={Kett House, Station Road},
             city={Cambridge},
             postcode={CB1 2JH},
             country={United Kingdom}}
  \affiliation[metu_ceng]{organization={Dept. of Computer Engineering, METU},
             addressline={Üniversiteler,  Çankaya},
             city={Ankara},
             postcode={06800},
             country={Türkiye}}
   \affiliation[ligm]{organization={LIGM, Ecole des Ponts, Univ Gustave Eiffel, CNRS},
             addressline={5 Boulevard Descartes-Champs s/ Marne},
             city={Marne-la-vallée Cedex 2},
             postcode={77454},
             country={France}}
    \affiliation[metu_romer]{organization={Center for Robotics and Artificial Intelligence (ROMER), METU},
             addressline={Üniversiteler,  Çankaya},
             city={Ankara},
             postcode={06800},
             country={Türkiye}}


 \author[tu_darmstadt,metu_ceng]{Baris Can Cam\corref{contrib}}
 \ead{bariscancam@gmail.com}
 \author[five_ai]{Kemal Oksuz\corref{contrib}}
 \author[metu_ceng]{Fehmi Kahraman}
 \author[ligm]{Zeynep Sonat Baltaci}
 \author[metu_ceng,metu_romer]{Sinan Kalkan\corref{sencontrib}}
 \author[metu_ceng,metu_romer]{Emre Akbas\corref{sencontrib}}
\cortext[contrib]{Equal contribution.}
\cortext[sencontrib]{Equal contribution for senior authorship.}

\begin{abstract}
This paper introduces Generalized Mask-aware Intersection-over-Union (GmaIoU) as a new measure for positive-negative assignment of anchor boxes during training of instance segmentation methods. Unlike conventional IoU measure or its variants, which only consider the proximity of anchor and ground-truth boxes; GmaIoU additionally takes into account the segmentation mask. This enables GmaIoU to provide more accurate supervision during training. We demonstrate the effectiveness of GmaIoU by replacing IoU with our GmaIoU in ATSS, a state-of-the-art (SOTA) assigner. Then, we train YOLACT, a real-time instance segmentation method, using our GmaIoU-based ATSS assigner. The resulting YOLACT based on the GmaIoU assigner outperforms (i) ATSS with IoU by $\sim 1.0-1.5$ mask AP, (ii) YOLACT with a fixed IoU threshold assigner by $\sim 1.5-2$ mask AP over different image sizes and (iii) decreases the inference time by $25 \%$ owing to using less anchors. Taking advantage of this efficiency, we further devise GmaYOLACT, a faster and $+7$ mask AP points more accurate detector than YOLACT. Our best model achieves $38.7$ mask AP at $26$ fps on COCO test-dev establishing a new state-of-the-art for real-time instance segmentation.
\end{abstract}
\end{frontmatter}
\section{Introduction}
\label{sec:Introduction}

Instance segmentation is a visual detection problem, where the goal is to detect and classify object instances by locating them using pixel-level segmentation masks. Instance segmentation task inherently involves the object detection task \cite{ObjectDetectionArXiv, ObjectDetectionIJCV, ObjectDetectionPAMI, ObjectDetectionNeurocomp}, which categorizes objects within an image while outlining their spatial location by using bounding-box representations. To handle the diversity of objects appearing at various locations, scales and numbers,  instance segmentation methods \cite{yolact,MaskRCNN,maskscoring,yolact-plus} commonly employ a dense set of object hypotheses, called anchors, to ensure a maximum coverage for objects. Anchors are typically represented by boxes of varying sizes and aspect ratios.  
To maximize coverage, a large number of anchors (e.g., $\sim 20k$ per image in YOLACT \cite{yolact} for images of size $550 \times 550$) need to be assigned to ground truths boxes. An anchor that is assigned to a ground-truth box is considered a positive example, and anchors that are not assigned to any ground-truth box are called negative examples. The process of matching anchors with ground-truth boxes is known as the \textit{assignment problem} \cite{ATSS,paa}.

\begin{figure}[hbt!]
    \centerline{
        \includegraphics[width=\textwidth]{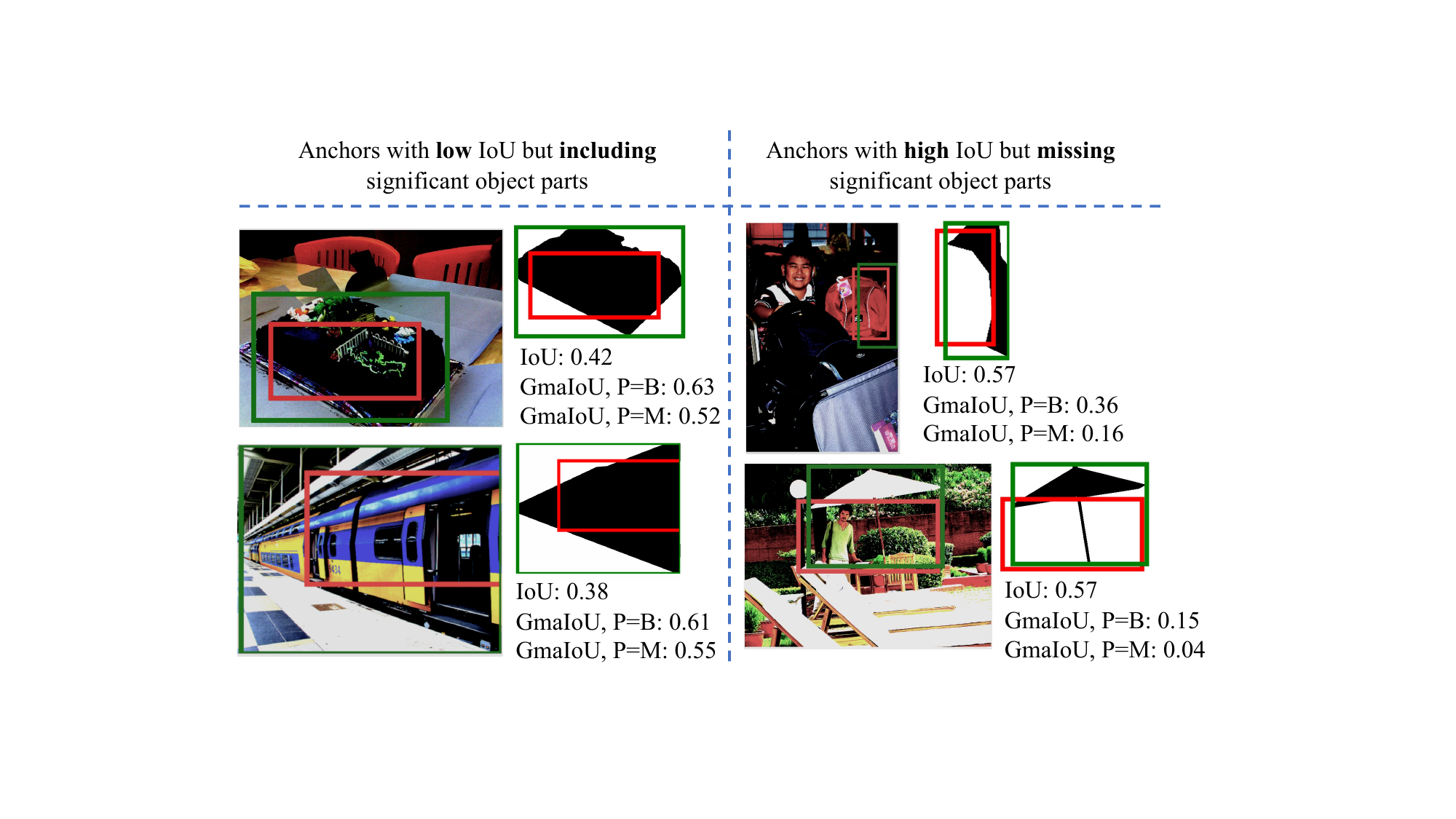}
    }
    \caption{Sample cases illustrating the need for Generalized mask-aware IoU (GmaIoU). Green boxes denote ground truth, red boxes are real anchors produced during the training. Left panel shows cases where the anchor covers a significant part of the object pixels but IoU is low (i.e. less than positive threshold of 0.50 for YOLACT). GmaIoU is higher than IoU for these cases, potentially correcting the assignment. Right panel shows cases where the anchor covers only a small part of the object pixels but IoU is high (so, anchors are positive). maIoU is lower than IoU, potentially correcting the assignment. Images are from COCO \cite{COCO}, GmaIoU is computed with $P=M$ as described in Section \ref{sec:miou}.\label{fig:Teaser}
} 
\end{figure}%

The assignment problem is commonly tackled using heuristic rules. One common way is to check the intersection over union (IoU) between an anchor and a ground-truth box, and employ a fixed threshold for deciding to match them 
\cite{yolact,MaskRCNN,retinamask}. 
Formally, an anchor, $\hat{B}$, can be assigned to a ground truth, $B$, when its IoU, defined as $\mathrm{IoU}(\hat{B}, B)=|\hat{B} \cap B|/|\hat{B} \cup B|$, exceeds a pre-determined threshold, $\tau$. In the case that an anchor $\hat{B}$ cannot be assigned to any of the ground-truths $B$ (i.e.  $\mathrm{IoU}(\hat{B}, B) < \tau,\;\,\forall B$), then $\hat{B}$ is assumed to be a background (i.e. negative) example. This conventional way of assigning the anchors with the ground truths can be referred as the ``fixed IoU threshold''  assignment strategy. As an alternative method, using an ``adaptive IoU threshold'', instead of a fixed one, has shown performance improvements in object detection 
\cite{ATSS,paa,noisyanchor}.
Still, most assignment methods rely on IoU as the de-facto proximity measure between ground truth boxes and anchors.

Despite its popularity, IoU has a certain drawback: The IoU between an anchor box and a ground truth box solely depends on their areas, thereby ignoring the shape of the object within the box, e.g. as provided by a segmentation mask. This may result in undesirable assignments due to counter-intuitively low or high IoU scores. For example, the IoU might be high, implying a positive anchor, but only a small part of the object is included in the anchor box; or the IoU may be low, implying a negative anchor, but a large part of the object is included in the anchor box. Fig. \ref{fig:Teaser} presents examples for such cases, arising due to objects with unconventional poses, occlusion and objects with articulated or thin parts. As we will show in our analysis  (in Section \ref{subsec:Analysis}, Fig. \ref{fig:MOB}), such examples tend to produce relatively larger loss values and adversely affect training.

In this paper, we introduce Generalized Mask-aware IoU (GmaIoU) as a novel measure for anchor assignment for instance segmentation. GmaIoU is based on exploiting the ground truth masks of the objects, normally used only for supervision through loss computation. Specifically, unlike the conventional IoU, which only compares an anchor box with a ground-truth box, GmaIoU compares an anchor box with a ground-truth box and mask pair. In IoU, all pixels within the boxes have equal importance, whereas in GmaIoU ground-truth mask pixels are promoted. Hence, GmaIoU produces an assignment/proximity score that is more consistent with the shape of the object (Fig. \ref{fig:Teaser}). 
Since a naive computation of GmaIoU is impractical due to dense pixel-wise comparisons considering the large number of anchors, we present an efficient algorithm resulting in a similar training time with using the conventional assignment methods based on IoU. YOLACT, a popular instance segmentation method, with GmaIoU-based ATSS assigner consistently improves the baseline IoU-based ATSS assigner by $\sim 1$ mask AP and standard YOLACT (i.e. fixed IoU threshold) by $\sim 2$ mask AP with a similar inference time of YOLACT. 


This paper extends our previous work \cite{maIoU} in three critical aspects. Firstly, we reformulate our novel mask-aware IoU definition for a more general case; between an arbitrary polygon (i.e., ground-truth masks) and a box. By this, we obtain GmaIoU which allows a switch between using extra box information ($P=B$), or using only the mask information ($P=M$). We show that the case $P=M$ provides consistent gains in terms of mask AP over the initially-proposed maIoU \cite{maIoU}, which corresponds to the special case of $P=B$ in the current generalized formulation. Secondly, we generalize the algorithm that consider ground truth masks efficiently during training for instance segmentation such that GmaIoU can be efficiently incorporated into the assignment method during training for both of its special cases ($P=B$ or $P=M$). We note that our algorithm is critical to leverage ground truth masks in both cases, as the brute force computation cannot be practically employed owing to the exploding runtime.  Finally, we improve our previous maYOLACT detector and develop  GmaYOLACT by incorporating recently proposed improvement strategies. The resulting GmaYOLACT detector performs instance segmentation in real-time with similar efficiency with the baseline YOLACT, but results in a significant $\sim 7$ mask AP and $\sim 10$ box AP improvement in terms of segmentation and detection performance respectively. We also note that GmaYOLACT detector outperforms our previous maYOLACT detector $\sim 1$ mask AP and $\sim 2.5$ box AP; demonstrating the effectiveness of our design in this revised version of the paper. 

\section{Related Work}
\label{sec:RelatedWork}
\begin{table*}
    \centering
    \small
    \caption{\small IoU Variants, their inputs and primary purposes. IoU variants assign a proximity measure based on the properties (prop.) of two inputs (Input 1 and Input 2). In practice, existing variants compare the inputs wrt. the same properties (i.e. either boxes or masks). Our Generalized Mask-aware IoU (GmaIoU) can uniquely compare a box with a box and a mask. With this, GmaIoU compares anchors (i.e. only box) with ground truths (box and mask) in order to provide better anchor assignment. *: GIoU is also used as a performance measure. }
    \resizebox{\textwidth}{!}{\begin{tabular}{|c|c|c|c|c|c|}
         \hline
         \multirow{2}{*}{IoU Variant}&
         \multicolumn{2}{c|}{Input 1 prop.} & \multicolumn{2}{c|}{Input 2 prop.} & Primary purpose as \\
         \cline{2-5}
         &Box&Mask&Box&Mask&proposed in the paper\\ \hline \hline
         Mask IoU \cite{COCO, Cityscapes}& \xmark & \cmark & \xmark & \cmark  & Performance measure  \\ \hline 
         Boundary IoU \cite{boundaryiou}& \xmark & \cmark & \xmark & \cmark & Performance measure \\ \hline
         Generalized IoU \cite{GIoULoss}  & \cmark &\xmark & \cmark &  \xmark & Loss function* \\ \hline
         Distance IoU \cite{DIoULoss} & \cmark &\xmark & \cmark &  \xmark & Loss function \\ \hline
         Complete IoU \cite{CIoULoss} &\cmark &\xmark & \cmark &  \xmark & Loss function \\ \hline \hline
         \textit{Generalized Mask-aware IoU (Ours)} & \cmark & \cmark & \cmark & \xmark & \textit{Assigner} \\ \hline
    \end{tabular}}
    \label{tab:iou_variants}
\end{table*}

\subsection{Deep Instance Segmentation} 

Deep learning based instance segmentation studies typically take a `detect and segment' approach, taking inspiration from and augmenting the deep object detection approaches. For example, the prominent instance segmentation model Mask R-CNN \cite{MaskRCNN} and its variations \cite{maskscoring,RSLoss} extended  Faster R-CNN \cite{FasterRCNN}, the infamous two-stage object detector, by adding a new branch for mask prediction. The mask prediction branch is trained simultaneously with the classification and localization branches. Over the years, a similar extension has been introduced over one-stage object detectors to obtain one-stage instance segmentation networks. Examples include YOLACT \cite{yolact} and YOLACT++ \cite{yolact-plus}, which are are based on a YOLO-like architecture;  PolarMask \cite{polarmask} and PolarMask++ \cite{PolarMask-plus}, which extended FCOS \cite{FCOS} for instance segmentation.

An alternative approach is to formulate instance segmentation directly as an instance classification problem, as proposed in SOLO variants \cite{solo,solov2}. In this approach, each cell in a grid generates an instance mask in parallel to predicting an object category. Another approach is to rely on transformers and predict a sparse set of instance segmentation masks without resorting to the NMS \cite{istr,li2022mask,yu2022soit}. We note that in this work we focus on the anchor-based instance segmentation methods, hence these alternative approaches are out of our scope.  



\subsection{Anchor Assignment in Instance Segmentation}

During training, anchor-based instance segmentation methods such as YOLACT \cite{yolact,yolact-plus} and Mask R-CNN variants \cite{MaskRCNN,maskscoring} require labeling anchors as `positive' (as an object hypothesis) or `negative' (belonging to background) by looking at their IoU values. In this approach, if an anchor's IoU with a ground truth object is larger than $\tau^+$, the anchor is labeled as a positive. Otherwise, i.e. if the anchor's IoU is less than $\tau^-$ (not necessarily equal to $\tau^+$), the anchor is labeled as a negative. An anchor whose maximum IoU with the ground truths is between $\tau^-$ and $\tau^+$ is considered as an outlier and ignored during training. 

As an example, Mask R-CNN uses $\tau^-=0.30$ and $\tau^+=0.70$ for the first stage (i.e. the region proposal stage), and $\tau^-=\tau^+=0.50$ for the second stage. In contrast, YOLACT variants \cite{yolact,yolact-plus} and RetinaMask \cite{retinamask} prefer $\tau^-=0.40$ and $\tau^+=0.50$.

The methods for assigning anchors in instance segmentation are identical to those used in deep learning-based object detection. Vo et al. \cite{neurcomp_assignment_review} proposed an in-depth analysis of anchor assignment and sampling strategies in deep learning-based object detection. 



\subsection{Adaptive Anchor Assignment Methods in Object Detection} 

Recent studies have shown that identifying positive and negative labels for anchors based on the anchor distribution performs better than using fixed IoU thresholds. A prominent study is ATSS \cite{ATSS}, wherein top-k anchors with highest IoU values are used to determine the IoU threshold for each ground truth (see also Section \ref{subsec:ATSS} for more details on ATSS). Another example is PAA \cite{paa} wherein a Gaussian Mixture Model is fit to Generalized IoU values for obtaining the distribution of positives and negatives for each ground truth. Fu et al. proposed a dynamic anchor assignment strategy \cite{adasgpm} which uses the Gaussian probabilistic distribution-based fuzzy similarity metric (GPM) and the adaptive dynamic anchor mining strategy (ADAS) to improve tiny object detection. The GPM measures the similarity between small bounding boxes and pre-defined anchors more accurately, while ADAS adjusts label assignment dynamically to better match the object distribution in the image. Similar approaches have been adopted by Li et al. \cite{noisyanchor} to dynamically label anchors based on their cleanliness and Ke et al. \cite{mal} to formulate anchor selection as a multiple instance learning problem. Although these studies have showcased promising results, they have been designed and evaluated for object detection and they do not utilize object masks in their assignment mechanisms.


\subsection{Other IoU Variants} 

IoU is a useful geometric measure for quantifying overlap between shapes, and over the years, IoU and its variants have been widely used in the literature for different purposes -- see Table \ref{tab:iou_variants} for a comparative summary. The table shows that IoU can be applied to not only boxes but also boundaries and masks as a measure of overlap. Moreover, we see that IoU or its variations with different forms of normalization (as in Generalized IoU \cite{GIoULoss}, Distance IoU \cite{DIoULoss}, Complete IoU \cite{CIoULoss}) can be used as loss functions for training localization branches of object detectors. 

The IoU variant that is most relevant to our contribution, i.e. Mask-aware IoU, is  Mask IoU \cite{COCO, Cityscapes}, which also depends on and uses masks. In contrast to Mask-aware IoU, Mask IoU is designed to measure the similarity between two masks (i.e., a predicted mask and the ground truth mask) during performance evaluation. Boundary IoU \cite{boundaryiou} is similar in that it considers a thin mask around object boundaries for calculating boundary-wise similarity between two shapes. 

Although these measures have yielded promising gains in their respective tasks, Mask-aware IoU has unique differences: First of all, Mask IoU and Boundary IoU only measure similarity between two masks, and therefore, they cannot compare a bounding box with another box and the mask in that box. Secondly, the other IoU measures are designed for quantifying similarity between boxes only, and therefore, they do not consider overlap at the shape level. Moreover, they have been used primarily as loss functions and not as an anchor assignment criterion.

\subsection{Comparative Summary.} 

Our coverage of related work above identifies two main gaps in anchor-based instance segmentation approaches: (i) They use IoU as the main criterion for labeling anchors and (ii) they rely on fixed IoU thresholds. 

To address the first limitation, we introduce GmaIoU as the first IoU measure that considers the ground truth mask while evaluating the overlap with an anchor box and the ground truth box -- see Table \ref{tab:iou_variants} for a comparison with existing measures. As for the second limitation, we incorporate our GmaIoU into the dynamic assignment strategy of ATSS \cite{ATSS}, which has provided significant gains in deep object detection. 

We replace the conventional IoU-based assignment strategy in instance segmentation by our GmaIoU-based ATSS assigner. Our novel assigner improves the performance of YOLACT both in terms of accuracy by $1.5-2.0$ mask AP and in terms of efficiency. We, then, utilize the efficiency gap enabled by our assigner. Specifically, we meticulously select improvement strategies and incorporate into the YOLACT such that its efficiency is preserved. In such a way, we obtain GmaYOLACT detector which has a similar inference time with YOLACT but significantly (around 7 mask AP) more accurate.

\section{Methodology}
\label{sec:miou}

Existing anchor assignment mechanisms rely heavily on the conventional IoU between the anchor and the ground truth boxes, thereby effectively ignoring the shapes of the objects. Here, in Section \ref{subsec:Analysis} we first demonstrate that among the anchors with similar box-level IoUs, the shape of the object within the anchor has an affect on learning, which is simply ignored by the existing IoU-based anchor assignment mechanisms. To address this gap, we design Generalized Mask-aware Intersection-over-Union (GmaIoU) in Section \ref{subsec:maIoU} as an IoU variant that takes into account the ground-truth masks, already available to supervise the instance segmentation methods, and use GmaIoU to assign anchors to ground-truths. A na\"ive computation of GmaIoU, which requires to process a large number of masks at each training iteration, is not efficient in terms of processing time. To overcome this, in Section \ref{subsec:Algorithm}, we present an algorithm to decrease the effect of the GmaIoU on the training time significantly. As a result, the resulting GmaIoU can be efficiently used by existing anchor assignment methods, which we incorporate into ATSS \cite{ATSS} as a SOTA anchor assignment strategy in Section \ref{subsec:ATSS}.


\begin{figure*}
    \centering
    \includegraphics[width=\textwidth]{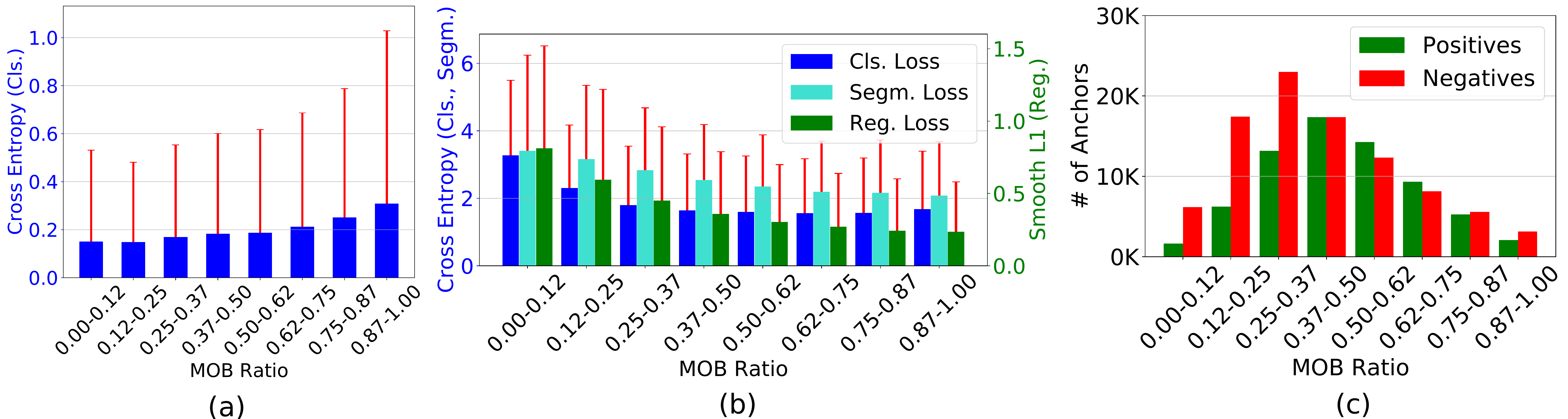}
    \caption{\small \textbf{(a,b)} Mean and standard deviation of loss values of negative anchors (a) and positive (b) anchors with \textbf{similar IoUs} (IoU between $[0.30-0.50]$ for negatives and $[0.50-0.70]$ for positives) over different MOB ratios for a \textbf{trained} YOLACT on COCO \textit{minival}. Red lines denote the standard deviation. Note that when the MOB ratio increases, the loss values  increase for negatives; however,  the loss values of all three sub-tasks (Cls.: classification, Segm.: segmentation, Reg.: regression) tend to decrease for positives.
    \textbf{(c)} The number of anchors for each MOB ratio is in the order of thousands.
    \label{fig:LossAnalysis}}
\end{figure*}

\subsection{The Mask-over-box Ratio and Observations}
\label{subsec:Analysis}

In this section, we demonstrate that the shape of an object inside an anchor, typically ignored by the conventional IoU, has an effect on how good this anchor is predicted by a segmentation model. Such an analysis requires quantifying how dense an anchor is in terms of the ground truth mask pixels. Accordingly, we first introduce a simple, but intuitive measure called the \textit{Mask-over-box (MOB) ratio}. Broadly speaking, the \textit{Mask-over-box (MOB) ratio} measured between a box and a mask is the ratio of the mask pixels falling into the box. More formally,

\begin{definition}
Mask-over-box (MOB) ratio between a box (i.e. anchor or ground truth), $\bar{B}$, and a mask (i.e. generally a ground-truth in our context), $M$, is the ratio between (i) the area corresponding to the intersection of the mask $M$ and the box $\bar{B}$, and (ii) the area of the $\bar{B}$ itself:
\begin{align}
\mathrm{MOB}(\bar{B},M)= \frac{|\bar{B} \cap M|}{|\bar{B}|}.     
\end{align}
\end{definition}

Since $|\bar{B}| \geq |\bar{B} \cap M|$, $\mathrm{MOB}(\bar{B},M) \in [0,1]$. In the case that a box $\bar{B}$ contains few pixels from $M$, $\mathrm{MOB}(\bar{B},M)$ will be lower, and it will be higher if $\bar{B}$ is dense with pixels from $M$. One specific configuration that we will look at is when $\bar{B}$ is the ground truth box of $M$. In that case, i.e., $\bar{B}=B$, $|M \cap B|=|M|$, and consequently, $\mathrm{MOB}(B,M)=|M|/|B|$. 

First, leveraging our MOB Ratio, we demonstrate that the shape of an object falling into an anchor has an effect on the quality of the segmentation. To do so, we adopt YOLACT, an anchor-based real-time instance segmentation method, with a ResNet-50 feature extractor. Specifically, during inference for each anchor, YOLACT predicts a classification score, a regressed box and a segmentation mask, and combining these outputs enables the model to have a segmentation mask with a classification label for each input anchor. During training, if an anchor matches with an object, i.e., it becomes a positive anchor, the losses for all three sub-tasks (classification, regression, and segmentation) are estimated for that anchor. If the anchor is deemed as negative, it contributes only to the classification loss. 

Considering these tasks and the label of the anchor, we want to see whether the mask pixels inside anchors with similar IoUs have an effect on the prediction quality. Consequently, we apply \textit{converged} YOLACT to the images in the COCO validation set and plot 
\begin{compactitem}
    \item the average classification loss values wrt. MOB ratio for negative anchors with similar IoUs (i.e., IoU between $[0.30-0.50]$) in Fig. \ref{fig:LossAnalysis}(a), and
    \item the average loss values for all tasks wrt. MOB ratios for positive anchors with similar IoUs (i.e., IoU between $[0.50-0.70]$) in Fig. \ref{fig:LossAnalysis}(b),
\end{compactitem}
both enabling us to  make the following crucial observation:

\begin{observation}
\textbf{The loss values of the anchors with similar IoUs \textit{is affected in all tasks} as MOB ratio changes for both positive and negative anchors.} Specifically, for negative anchors, the loss increases with increasing MOB (Fig. \ref{fig:LossAnalysis}(a)). 
However, it is the other way around for the positive anchors (Fig. \ref{fig:LossAnalysis}(b)).  
Also, the numbers of anchors with larger losses are in the order of thousands in all cases (Fig. \ref{fig:LossAnalysis}(c)). 
These results depict that the error of the model for an anchor is related to the ratio of mask pixels covered by that anchor, which is completely ignored by the conventional IoU-based assigners.
\end{observation}


Next, we conduct a similar analysis on the ground-truth boxes (Fig. \ref{fig:MOB}(a)), which leads us to our second observation. 

\begin{observation} \textbf{Similar to anchors, MOB ratios of the ground truth boxes also vary significantly}. Note that a notable ratio of the ground-truth boxes has low MOB ratios (i.e. for $30\%$ of the ground truths, MOB ratio is less than $0.50$). 
\end{observation}
To summarize, ground-truth boxes incur greatly varying MOB ratios, which cannot be captured by conventional IoU-based assignment methods as they do not make use of masks. Therefore, there is room for an alternative IoU measure that takes into account masks of objects.

\subsection{Generalized Mask-aware Intersection-over-Union}
\label{subsec:maIoU}

\paragraph{Intuition} The main intuition behind Generalized mask-aware IoU (GmaIoU) is to take into account the ground-truth masks while computing the overlap between the anchor and the ground truth object. 
Subsequently, this proximity is used in the assignment of anchors to ground-truth objects, which corresponds to labeling them as positive or negative. 
Particularly, if the anchor box includes more on-mask pixels, then it will have a larger assignment score measured by our GmaIoU. In the initial version of this measure coined as Mask-aware IoU (maIoU) \cite{maIoU}, we achieved this by defining the \textit{energy} of a ground truth by the area of its bounding box, i.e., $|B|$ and distributed the contributions of the off-mask pixels uniformly over the on-mask pixels in $B$. Here, for the sake of generalization, we instead assume a polygon $P$ for each ground that comprises its energy as $|P|$. As we will discuss, this perspective enables us to formulate not only maIoU,  as well as the IoU between a box and a mask, thereby increasing the generalizability of our approach. 

\paragraph{Derivation} We first start our derivation with the conventional intersection formulation, which does not differentiate between on-mask and off-mask pixels. Specifically, denoting the contribution of on-mask pixels by $w_m$ and that of off-mask pixels by $w_{\overline{m}}$, the intersection between $B$ and $\hat{B}$ ($\mathcal{I}(B, \hat{B})$) can be written as:
\begin{equation}
\mathcal{I}(B, \hat{B}) = \sum_{i\ \in\ B \cap \hat{B}} w = \sum_{i\ \in\ \hat{B} \cap M} w_m +  \sum_{i\ \in\ \left(\hat{B} \cap B - \hat{B} \cap M \right )} w_{\overline{m}}, \label{eqn:IoU}
\end{equation}
which effectively does not make use of the mask $M$ since $w=w_m=w_{\overline{m}}=1$ for IoU as also illustrated in Fig. \ref{fig:MOB}(b).

In GmaIoU, different from the IoU, we basically discard the contribution of off-mask pixels to the IoU computation by using $w_{\overline{m}}=0$. As for the contribution of on-mask pixels ($w_{m}$), as mentioned while providing the intuition, we introduce $P$ as the polygon comprising the total energy of the ground truth box and  
effectively distribute the energy of all pixels in $P$ uniformly over on-mask pixels, yielding $w_m=|P|/|M|$. 
With these weights, the generalized mask-aware intersection, $\mathrm{Gma}\mathcal{I}$ is defined by extending Eq. \ref{eqn:IoU}:
\begin{align}
     \mathrm{Gma}\mathcal{I} (\hat{B}, P, M) & = \sum_{i\ \in\ \hat{B} \cap M} w_m +  \cancelto{0}{\sum_{i\ \in\ \left( \hat{B} \cap P-\hat{B} \cap M \right )} w_{\overline{m}}} \\
     &= w_m |\hat{B} \cap M| = \frac{ |P|}{|M|}|\hat{B} \cap M|.
\end{align}
Similarly, we extend the conventional union defined as $|B|+|\hat{B}|-|B \cap \hat{B}|$ using these new weights to obtain the generalized mask-aware union ($\mathrm{Gma}\mathcal{U} (\hat{B}, P, M)$) as follows:
\begin{align}
    \mathrm{Gma}\mathcal{U} (\hat{B}, P, M) = &  |P| + (|\hat{B}|-|\hat{B} \cap P| + \mathrm{Gma}\mathcal{I}(\hat{B}, P, M)) \nonumber \\
            & - \mathrm{Gma}\mathcal{I} (\hat{B}, P, M) \\
    = & |\hat{B} \cup P|.
\end{align}
Finally, Generalized Mask-aware IoU is simply the ratio between Generalized Mask-aware Intersection ($\mathrm{Gma}\mathcal{I} (\hat{B}, P, M)$) and Generalized Mask-aware Union ($\mathrm{Gma}\mathcal{U} (\hat{B}, P, M)$):
\begin{align}
    \label{eq:Definition1}
    \mathrm{GmaIoU} (\hat{B}, P, M) = \frac{\mathrm{Gma}\mathcal{I} (\hat{B}, P, M)}{\mathrm{Gma}\mathcal{U} (\hat{B}, P, M)} =  \frac{|P|}{|M|}\frac{|\hat{B} \cap M|}{|\hat{B} \cup P|}.
\end{align}
To provide more intuition, we next discuss two specific settings of Eq. \ref{eq:Definition1}, which we will also employ in the rest of the paper.
%
\begin{figure*}
        \centering
        \begin{tabular}{@{}c@{}}
                \includegraphics[width=0.40\textwidth]{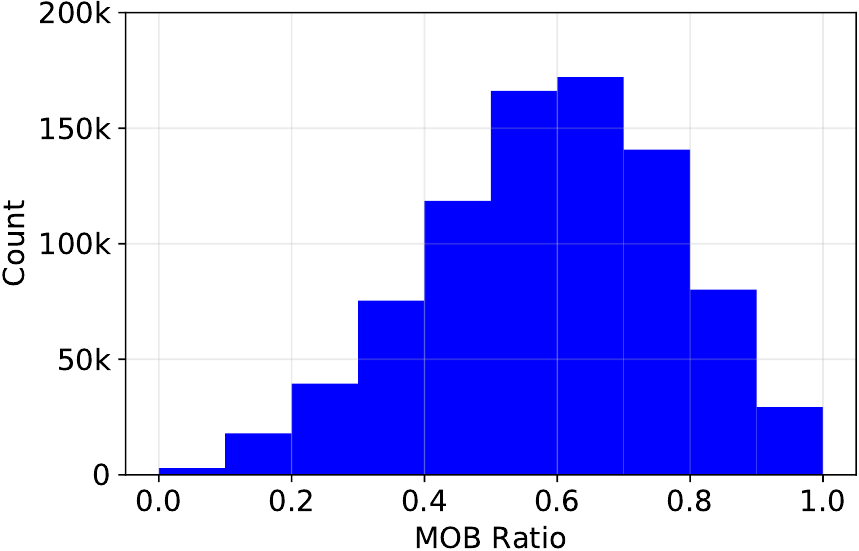}\\ [\abovecaptionskip]
        \footnotesize (a)
        \end{tabular}
        \begin{tabular}{@{}c@{}}
                \includegraphics[width=0.30\textwidth]{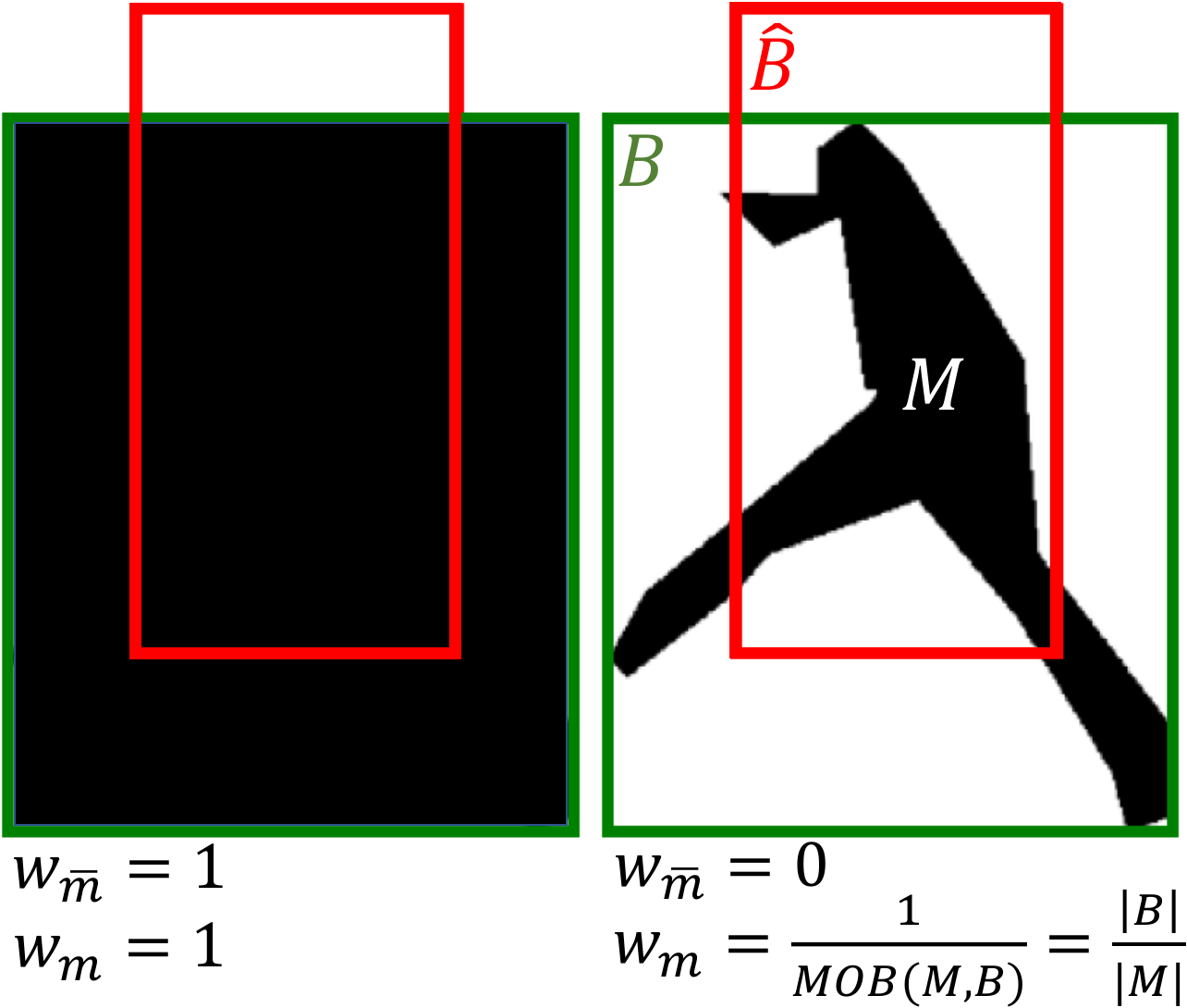}\\ [\abovecaptionskip]
        \footnotesize (b)
        \end{tabular}
        \caption{(a) The distribution of MOB ratios of the ground truths on COCO training set. (b) How IoU and maIoU weights the pixels in ground truth box (see Fig. \ref{fig:Teaser} top-left example for the image of this example). While IoU does not differentiate among on-mask ($w_m$) and off-mask ($w_{\overline{m}}$) pixels, our maIoU sets $w_{\overline{m}}=0$ and weights each on-mask pixel by inverse MOB ratio considering object mask $M$.}
        \label{fig:MOB}
\end{figure*}

\begin{figure*}
    \begin{tabular}{@{}c@{}}
        \includegraphics[width=0.32\textwidth]{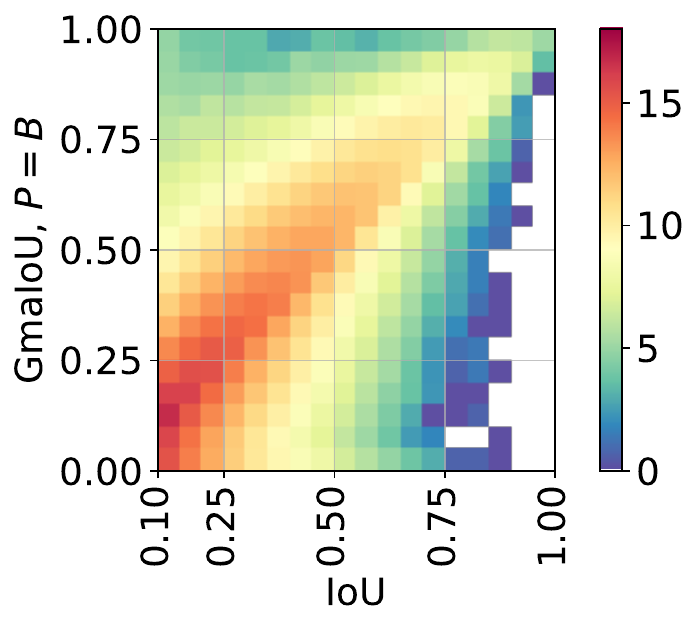} \\ [\abovecaptionskip]
        \footnotesize (a)
    \end{tabular}
    \begin{tabular}{@{}c@{}}
        \includegraphics[width=0.32\textwidth]{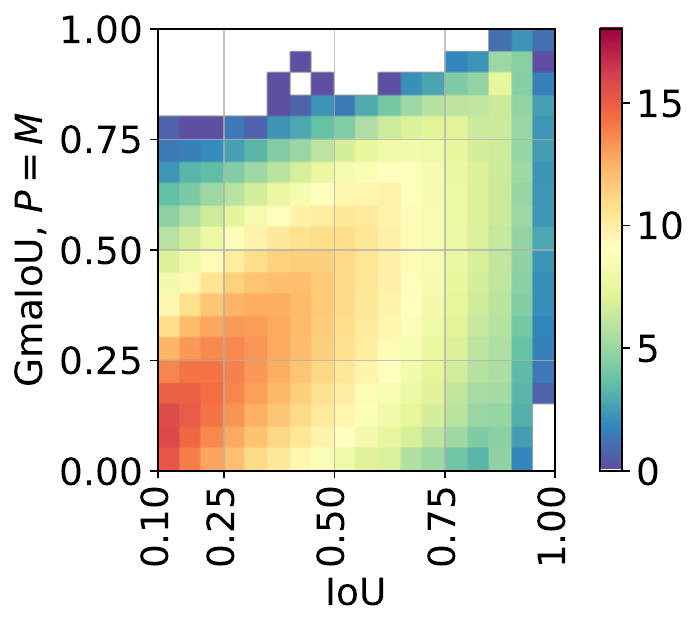}\\ [\abovecaptionskip]
        \footnotesize (b)
    \end{tabular}
    \begin{tabular}{@{}c@{}}
        \includegraphics[width=0.32\textwidth]{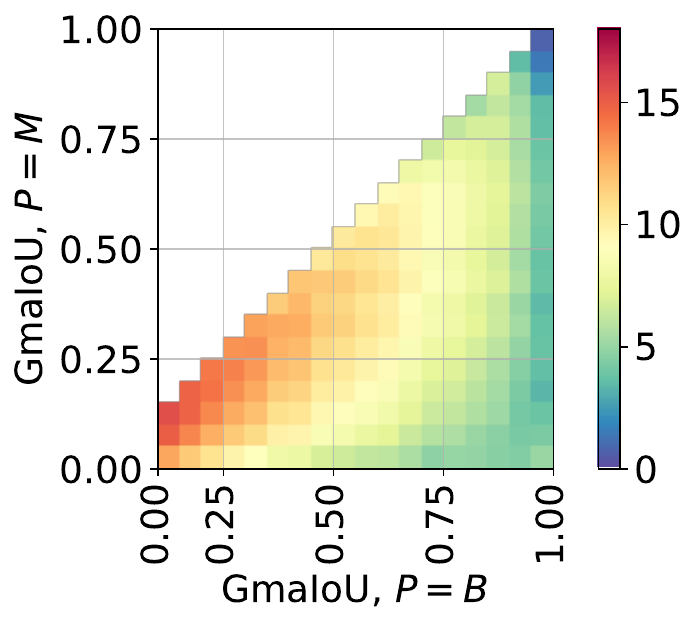}\\ [\abovecaptionskip]
        \footnotesize (c)
    \end{tabular}
    \caption{Anchor count distribution (in log-scale) of (a) IoU vs GmaIoU, $P=B$; (b) IoU vs GmaIoU, $P=M$; and (c) GmaIoU, $P=B$ vs GmaIoU, $P=M$. While $\mathrm{IoU}$ \& high-$\mathrm{GmaIoU}$ positively correlate for both $P=B$ and $P=M$, there are quite a few examples with low-$\mathrm{IoU}$ \& high-$\mathrm{GmaIoU}$ and vice versa as shown in (a) and (b). (c) empirically demonstrates that $P=B$ is an upper bound for $B=M$ case of GmaIoU.}
    \label{fig:histograms}
\end{figure*}

\paragraph{Special Cases of $P$} Here, we show that setting $P=B$ corresponds to maIoU, and the case where $P=M$ simply gives the IoU between the ground truth mask $M$ and the anchor $\hat{B}$. 

\textit{Case 1.} If $P=B$, then Eq. \ref{eq:Definition1} reduces to maIoU \cite{maIoU}:
\begin{align}
    \label{eq:Definition2}
    \frac{|B|}{|M|}\frac{|\hat{B} \cap M|}{|\hat{B} \cup B|}=  \frac{1}{\mathrm{MOB}(B,M)}\frac{|\hat{B} \cap M|}{|\hat{B} \cup B|}=\mathrm{maIoU} (\hat{B}, B, M).
\end{align}

\textit{Case 2.} If $P=M$, then Eq. \ref{eq:Definition1} reduces to the IoU between the ground-truth mask and the anchor box:
\begin{align}
    \label{eq:Definition3}
    \frac{|M|}{|M|}\frac{|\hat{B} \cap M|}{|\hat{B} \cup M|}= \frac{|\hat{B} \cap M|}{|\hat{B} \cup M|}=\mathrm{IoU} (\hat{B}, M).
\end{align}

Therefore, our GmaIoU unifies existing IoU variants incorporating mask information into the definition of the proximity measure; offering a generalized perspective on those.

\paragraph{Interpretation} For both $\mathrm{GmaIoU} (\hat{B}, B, M)$ and $\mathrm{IoU} (\hat{B}, M)$, the range is between $0$ and $1$ and they are higher-better measures similar to the conventional IoU. Their key difference is that our $\mathrm{GmaIoU} (\hat{B}, B, M)$ incorporates the overlap of the box with the mask as illustrated in Fig. \ref{fig:MOB}(b) for its $P=B$ case. Besides, Fig. \ref{fig:histograms}(a,b) demonstrates the relationship between the conventional IoU and $\mathrm{GmaIoU} (\hat{B}, B, M)$ using the anchors (in log-scale) of YOLACT on COCO minival: As would be expected, $\mathrm{IoU}$ and $\mathrm{GmaIoU}$ are positively correlated in both cases. However, we observe quite a number of anchors with low $\mathrm{IoU}$ but high $\mathrm{GmaIoU}$ and vice versa. This implies that the anchor assignment methodologies, as in the case of our $\mathrm{GmaIoU}$, considering mask are different than the ones that are ignoring them. In the following we explore the relationship between two different configurations of our $\mathrm{GmaIoU} (\hat{B}, B, M)$.
\begin{corollary}
\label{corr:m_vs_b}
$\mathrm{GmaIoU} (\hat{B}, B, M) \geq \mathrm{GmaIoU} (\hat{B}, M, M)$.
\end{corollary}
\begin{proof}
    \label{eq:proof_corr}
    Before proving the corollary, let us first rewrite Equation \eqref{eq:Definition1} as follows to make the proof easier:
\begin{align}
    \label{eq:corr_def2}
    \mathrm{GmaIoU}(\hat{B}, P, M) = \frac{|P|}{|\hat{B} \cup P|} \cancelto{\mathrm{constant}\, \mathrm{wrt.}\, P}{\frac{|\hat{B} \cap M|}{|M|}},
\end{align}
Here, we see that changing $P$ in GmaIoU only affects $\frac{|P|}{|\hat{B} \cup P|}$ and therefore, it is sufficient if we just focus on how this expression changes when $P$ changes: To prove the corollary, we extend the polygon $P$ by a single pixel $x$ (as changing $P$ from $M$ to $B$ grows the polygon) and derive the resulting $|P|/|\hat{B} \cup P|$: Let us use $P^{+}=P\cup x$ such that $x \notin P$ is a single pixel. 

There are two different cases for $x$:\\
    \par \underline{\textbf{Case (i): $x\in \hat{B}$}}\\
        \begin{align}
            \frac{|P^+|}{|\hat{B}\cup P^+|} &= \frac{|P|+1}{|\hat{B}|+(|P|+1) - (|\hat{B}\cap P)| + 1}, \\
            &= \frac{|P|+1}{|\hat{B} \cup P|} > \frac{|P|}{|\hat{B} \cup P|}.
        \end{align}
    \par \underline{\textbf{Case (ii): $x \notin  \hat{B}$}}\\
        \begin{align}
            \frac{|P^+|}{|\hat{B} \cup P^+|} &= \frac{|P|+1}{|\hat{B}|+(|P|+1) - |\hat{B}\cap P|},\\
            &=\frac{|P|+1}{|\hat{B} \cup P| + 1} \geq \frac{|P|}{|\hat{B} \cup P|}.
        \end{align}
Consequently, we can express the relationship in $M$ and $B$ as: $B = M \cup M'$ such that $|M| = 0$ or $|M| \geq 1$. If $|M| = 0$, then $B = M$ and  $\mathrm{GmaIoU} (\hat{B}, B, M) = \mathrm{GmaIoU} (\hat{B}, M, M)$. As for $|M| \geq 1$, considering two cases above $\mathrm{GmaIoU} (\hat{B}, B, M) \geq \mathrm{GmaIoU} (\hat{B}, M, M)$ holds; completing the proof.
\end{proof}

A visual interpretation of the comparison between these two measures, formally shown in Corollary \ref{corr:m_vs_b}, is shown in Figure \ref{fig:histograms}(c) where we see that two different settings of GmaIoU correspond to two different proximity measures. This flexibility of GmaIoU enables us to investigate how to consider ground truth masks for assigning the anchors to ground truth boxes.

%

\subsection{The Challenge to Utilize Masks for Assignment and An Efficient Computation Algorithm for GmaIoU}
\label{subsec:Algorithm}
On each input image, a typical anchor-based instance segmentation method defines thousands of anchor boxes (e.g., $\sim 19.2K$ for YOLACT for images with a size of $550 \times 550$) in an effort to  capture all ground-truth objects in the image, which are in different locations, scales and aspect ratios. Accordingly, an efficient mechanism assigning these large number of anchors with the ground truth objects is necessary; otherwise the training time can easily increase in orders of magnitude. Consequently, devising an efficient algorithm is also a major challenge while incorporating the masks into the assignment algorithm considering the fact that the masks of the objects can have arbitrary shapes unlike their bounding boxes. Therefore, in this section, we introduce an algorithm to efficiently compute GmaIoU for its $P=B$ and $P=M$ cases corresponding to $\mathrm{maIoU} (\hat{B}, B, M)$ and $\mathrm{IoU} (\hat{B}, M)$ respectively as outlined in Section \ref{subsec:maIoU}.

We first investigate the additional terms to compute $\mathrm{maIoU} (\hat{B}, B, M)$ and $\mathrm{IoU} (\hat{B}, M)$ (Eq. \eqref{eq:Definition3}) compared to the conventional IoU relying only on bounding boxes. $\mathrm{maIoU} (\hat{B}, B, M)$, as defined in Eq. \ref{eq:Definition2} requires (i) $|M|$, the number of mask pixels and (ii) $|\hat{B} \cap M|$, the number of mask pixels in the intersecting with the anchor box $\hat{B}$. As for $\mathrm{IoU} (\hat{B}, M)$ , $|\hat{B} \cap M|$ as the nominator of Eq. \eqref{eq:Definition3} is again necessary. The union expression in the denominator of Eq. \eqref{eq:Definition3} can be rewritten as $|\hat{B} \cup M|=|\hat{B}|+|M|-|\hat{B} \cap M|$; therefore another additional term is $|M|$ similar to $\mathrm{maIoU} (\hat{B}, B, M)$. To summarize; finding a way to compute $|M|$ and $|\hat{B} \cap M|$ efficiently will enable us to use both $\mathrm{maIoU} (\hat{B}, B, M)$ and $\mathrm{IoU} (\hat{B}, M)$ for anchor assignment.

To do so, we leverage \textit{integral images} \cite{IntegralImage} on the binary ground truth masks, $M$. In particular, given $M$ corresponding to an object with a size of $m \times n$, we obtain the integral image of $M$ denoted by $\Theta^M$ which is an $(m+1) \times (n+1)$ matrix. As a feature of integral image, $\Theta^M$ represents the total number of mask pixels above and to the left of each pixel. Then, given $\Theta^M$ and denoting its $(i,j)^{th}$ element by $\Theta^M_{i,j}$, we can recover the two terms that we are interested in to compute $\mathrm{maIoU} (\hat{B}, B, M)$ and $\mathrm{IoU} (\hat{B}, M)$ as follows:
\begin{compactitem}
    \item $|M|=\Theta^M_{m+1,n+1}$, that is the last element stores the total number of mask pixels, and
    \item $|\hat{B} \cap M|=\Theta^M_{x_2+1,y_2+1}+\Theta^M_{x_1,y_1}-\Theta^M_{x_2+1,y_1}-\Theta^M_{x_1,y_2+1}$ where the pairs $(x_1, y_1)$ and $(x_2,y_2)$ determines the top-left and bottom-right points of the anchor box $\hat{B}$ such that $x_2>x_1$ and $y_2>y_1$.
\end{compactitem}
Therefore, having computed the integral image for each object, obtaining $|M|$ and $|\hat{B} \cap M|$ require only constant time; enabling us to incorporate both $\mathrm{maIoU} (\hat{B}, B, M)$ and $\mathrm{IoU} (\hat{B}, M)$ into the anchor assignment procedure. We present the algorithm to compute these quantities in Alg. \ref{alg:MaskAwareIoU}.
\begin{algorithm}
\caption{The algorithm for efficiently calculating Generalized mask-Aware IoU for $P=B$ ($\mathrm{maIoU} (\hat{B}, B, M)$) or $P=M$ ($\mathrm{IoU} (\hat{B}, M)$). \label{alg:MaskAwareIoU}}
\footnotesize
\begin{algorithmic}[1]
\Procedure{GeneralizedMaskAwareIoU}{$\hat{B}, P, M$}
\State $\Theta^M \leftarrow $ integral image of $M$ st. $\Theta^M_{i,j}$ is $(i,j)^{th}$ element of $\Theta^M$ 
\State $|M|=\Theta^M_{m+1,n+1}$ 
\State $|\hat{B} \cap M|=\Theta^M_{x_2+1,y_2+1}+\Theta^M_{x_1,y_1}-\Theta^M_{x_2+1,y_1}-\Theta^M_{x_1,y_2+1}$
\If{$P=B$}
\State $\mathrm{MOB}(B,M)= |M|/|B|$ for ground truth $B$
\State Compute $|\hat{B} \cup B|$
\State \textbf{return} $\mathrm{maIoU} (\hat{B}, B, M)$ (Eq. \eqref{eq:Definition2})
\ElsIf{$P=M$}
\State Compute $|\hat{B}|$
\State $|\hat{B} \cup M|=|\hat{B}|+|M|-|\hat{B} \cap M|$
\State \textbf{return} $\mathrm{IoU} (\hat{B}, M)$ (Eq. \eqref{eq:Definition3})
\EndIf
\EndProcedure
\end{algorithmic}
\end{algorithm}



\subsection{Incorporating GmaIoU into ATSS Assigner}
\label{subsec:ATSS}
Adaptive Training Sample Selection (ATSS) \cite{ATSS} is a common SOTA anchor assignment procedure used in object detectors. Among its benefits compared to using a fixed IoU assigner are (i) ATSS yields better performance in object detectors; and (ii) broadly speaking, it simplifies the anchor design thanks to using a single anchor from each location on the image instead of requiring up to nine anchors per location by the fixed IoU threshold assigner \cite{FocalLoss}. Basically, ATSS assigner consists of three steps: (i) considering the distance between the centers of the anchor and the ground truth box, top-$k$ anchors are selected as ``candidates'' for positive assignment where typically $k=9$; (ii) an adaptive IoU threshold is computed based on the statistics of these candidates and the anchors with IoU less than this adaptive IoU are removed from the candidate set; and finally (iii) the candidate anchors with centers not residing in the ground truth bounding box are filtered out. Having completed these three steps, the remaining anchors in the candidate set are the positive ones (i.e., match with an object) and all other anchors are negatives to be predicted as background. Incorporating our GmaIoU into ATSS is straightforward: In step (ii), we simply replace the IoU-based adaptive thresholding by our GmaIoU-based adaptive thresholding.

\section{Experiments}
\label{sec:Experiments}

 In this section, we evaluate the effectiveness of our GmaIoU by incorporating it into YOLACT \cite{yolact}, an anchor-based state-of-the-art (SOTA) real-time instance segmentation method with a good trade-off between efficiency and performance. Specifically, Section \ref{subsec:Ablation} presents ablation experiments comparing different configurations of GmaIoU and IoU variants.  Section \ref{subsec:GmaYOLACT} builds GmaYOLACT detector by using several improvement strategies for accuracy by keeping the real-time nature of the method. Finally, Section \ref{subsec:SOTA} compares our approach with SOTA instance segmentation methods. 

\subsection{Experiment and Implementation Details}
\noindent \textbf{Dataset.} We employ the commonly-used COCO \textit{trainval} set \cite{COCO} (115K images) for training and COCO \textit{minival} set (5K images) unless otherwise stated.

\noindent \textbf{Performance Measures.} In general, we rely on the common AP-based performance measures (higher-better) including (i) COCO-style AP denoted by AP; (ii) APs where true positives are validated from IoUs of 0.50 and 0.75 ($\mathrm{AP_{50}}$ and $\mathrm{AP_{75}}$); (iii) APs designed for small, medium and large objects ($\mathrm{AP_{S}}$, $\mathrm{AP_{M}}$ and $\mathrm{AP_{L}}$ respectively). Furthermore, we leverage optimal Localisation Recall Precision (oLRP) Error \cite{LRP, LRP_tpami} (lower-better) enabling us to provide more insights. 

\noindent \textbf{Implementation Details.} We incorporate our GmaIoU assigner and GmaYOLACT detector into the mmdetection framework \cite{mmdetection} by following the standard configurations of YOLACT and ATSS anchor assignment strategy.
%
%
While using ATSS with any IoU variant, we find it useful to keep the classification, box regression and semantic segmentation loss weights ($1.0$, $1.5$ and $1.0$ respectively) as they are but we increase the mask prediction loss weight from $6.125$ to $8$. 
%
%
Unless otherwise stated, we train all models with a batch size of $32$ using $4$ Nvidia GeForce GTX Titan X GPUs.
%
We use ResNet-50 \cite{ResNet} as the backbone; resize images to one of the $400 \times 400$, $550 \times 550$ or $700 \times 700$ during training and testing following Bolya et al. \cite{yolact}.

\begin{table*}
    \centering
    \small
    \setlength{\tabcolsep}{0.35em}
    \caption{Comparison of different assigners and IoU variants on YOLACT. Considering the shapes of the objects, our ATSS w. our maIoU consistently outperforms its counterparts.}
    \label{tab:minival}
    \resizebox{\textwidth}{!}{\begin{tabular}{|c|c|c|c|c|c|c|c|c|} \hline
        Scale & Assigner & $\mathrm{AP} \uparrow$ & $\mathrm{AP_{50}} \uparrow$ & $\mathrm{AP_{75}} \uparrow$ & $\mathrm{AP_{S}} \uparrow$ & $\mathrm{AP_{M}} \uparrow$ & $\mathrm{AP_{L}} \uparrow$ & $\mathrm{oLRP} \downarrow$ \\ \hline \hline
    \multirow{5}{*}{400}&fixed IoU threshold&$24.8$&$42.4$&$25.0$&$\mathbf{7.3}$&$26.0$&$42.0$&$78.3$ \\
    &ATSS w. IoU&$25.3$&$43.5$&$25.5$&$6.8$&$27.3$&$43.8$&$77.7$ \\
    &ATSS w. DIoU&$25.4$&$43.6$&$25.2$&$7.2$&$27.1$&$43.4$&$77.7$ \\
    &ATSS w. GIoU&$25.1$&$42.7$&$25.3$&$7.0$&$26.8$&$41.8$&$78.0$ \\ \cline{2-9}
    &ATSS w. GmaIoU, $P=B$ (Ours)&$26.1$&$\mathbf{44.3}$&$26.3$&$7.2$&$28.0$&$44.3$&$\mathbf{77.1}$\\
    &ATSS w. GmaIoU, $P=M$ (Ours) &$\mathbf{26.2}$&$44.0$&$\mathbf{26.6}$&$5.6$&$\mathbf{28.6}$&$\mathbf{48.1}$&$\mathbf{77.1}$ \\
    \hline
    \hline 
    \multirow{5}{*}{550} &fixed IoU threshold& $28.5$ & $47.9$ & $29.4$ & $11.7$ & $31.8$ & $43.0$ & $75.2$ \\
    & ATSS w. IoU&$29.3$& $49.2$ & $30.2$ & $11.1$ & $33.0$ & $44.5$ & $74.5$ \\
    &ATSS w. DIoU&$29.5$&$49.5$&$30.1$&$11.7$&$33.2$&$44.9$&$74.4$ \\
    &ATSS w. GIoU&$29.1$&$48.6$&$30.0$&$\mathbf{12.0}$&$32.2$&$43.3$&$74.7$ \\ \cline{2-9}
    &ATSS w. GmaIoU, $P=B$ (Ours)& $30.4$& $50.3$&$31.4$ &$11.5$&$33.9$&$46.3$&$73.7$ \\
    & ATSS w. GmaIoU, $P=M$ (Ours)& $\mathbf{30.7}$& $\mathbf{50.4}$ & $\mathbf{31.8}$ & $10.5$ & $\mathbf{34.0}$ & $\mathbf{50.2}$ & $ \mathbf{73.5}$ \\
    \hline
    \hline
    \multirow{5}{*}{700} &fixed IoU threshold&$29.7$&$50.0$&$30.4$&$14.2$&$32.8$&$43.7$&$74.3$\\
    & ATSS w. IoU&$30.8$&$51.8$&$31.2$&$14.1$&$35.0$&$44.0$&$73.3$\\
    &ATSS w. DIoU&$30.9$&$51.9$&$31.7$&$14.0$&$35.4$&$44.0$&$73.3$ \\
    &ATSS w. GIoU&$30.1$&$50.7$&$31.0$&$14.0$&$33.8$&$43.1$&$74.0$ \\ \cline{2-9}
    &ATSS w. GmaIoU, $P=B$ (Ours)&$31.8$&$52.8$&$32.8$&$\mathbf{14.7}$&$35.6$&$45.7$&$72.5$\\ 
    &ATSS w. GmaIoU, $P=M$ (Ours) &$\mathbf{32.2}$&$\mathbf{53.2}$&$\mathbf{33.1}$&$13.2$&$\mathbf{35.9}$&$\mathbf{49.4}$&$\mathbf{72.3}$ \\
    \hline
    
    \end{tabular}}
\end{table*}

\subsection{Ablation Experiments}
\label{subsec:Ablation}

In this section, we demonstrate that GmaIoU consistently outperforms assigners using different IoU variants which ignore the shape of the object. Furthermore, we show that our Alg. \ref{alg:MaskAwareIoU} enables the efficient calculation of GmaIoU making it feasible to use during training.

\noindent \textbf{Using ATSS with IoU variants.} We first build a stronger baseline in which we replace the fixed IoU threshold assigner in YOLACT by ATSS with IoU and its different variants including DIoU \cite{DIoULoss} and GIoU \cite{GIoULoss}. We observe in Table \ref{tab:minival} that using ATSS outperforms the conventional fixed IoU assigner by $0.6 - 1.2$ mask AP across different scales.

\noindent \textbf{Using ATSS with our GmaIoU.} We now use our GmaIoU with ATSS and enable the assignment method to consider the shapes of the objects as opposed to the existing anchor assignment methods. We test both configurations of GmaIoU ($P=B$ and $P=M$) introduced in Section \ref{sec:miou} while extending maIoU to GmaIoU. We note that the case $P=M$ is derived and computed thanks to our generalization and efficient algorithm respectively while the case $P=B$ is simply maIoU \cite{maIoU}. Table \ref{tab:minival} suggests that our generalization is, in fact, useful compared to relying on the standard maIoU ($P=B$) observing that $P=M$ outperforms $P=B$ across all scales. Furthermore, our GmaIoU with $P=M$ (i) outperforms the fixed IoU assigner by $1.4$, $2.2$, and $2.5$ mask APs for 400, 500 and 700 scales respectively, (ii) improves ATSS w. IoU variants by $\sim 1.0$ mask AP in all scales. Major improvements of GmaIoU are observe especially on (1) models trained by larger scales (700 vs. 400 in Table \ref{tab:minival}) and (2) larger objects ($\mathrm{AP_{L}}$ vs. $\mathrm{AP_{S}}$). These results are inline with the high-level idea of our GmaIoU (Figure \ref{fig:Teaser}) because the shape of the object becomes more important than its bounding box when its size increases.

\noindent \textbf{Computing GmaIoU Efficiently.} `Brute-force' computation of GmaIoU for every anchor-ground truth pair is impractical during training, as it would take $\sim 3$ months to train a single model with $41.89$ sec/iteration rate. Table \ref{tab:runtime} suggests that our efficient algorithm (Alg. \ref{alg:MaskAwareIoU}) yields a substantial gain ($\sim 70\times$) in the average iteration time to $0.59$ sec/iteration and enables the assignment algorithms to consider the shape of the object. Through this, we obtain an average iteration time similar to other standard assigners (Table \ref{tab:runtime}). We note that the computation of the case $P=M$ is very similar to the case $P=B$ for GmaIoU; hence the same conclusion follows for $P=M$ as well.

\begin{table*}
\RawFloats
\parbox{.38\linewidth}{
    \centering
    \small
    \setlength{\tabcolsep}{0.3em}
    \caption{Avg. iteration time ($t$) of assigners. While brute force maIoU computation is inefficient (Alg. \ref{alg:MaskAwareIoU} is \xmark); our Alg. \ref{alg:MaskAwareIoU} decreases $t$ by $\sim 70 \times$ and has similar $t$ with Fixed IoU Thr. and ATSS w. IoU.}
    \label{tab:runtime}
    \resizebox{.38\textwidth}{!}{\begin{tabular}{|c|c|c|} \hline
    Assigner& Alg. \ref{alg:MaskAwareIoU}& $t$ (sec.) \\ \hline
    Fixed IoU Thr.&N/A&$0.51$\\
    ATSS w. IoU&N/A&$0.57$ 
    \\
    ATSS w. GmaIoU, $P=B$&\xmark&$41.89$
    \\ \hline
    ATSS w. GmaIoU, $P=B$&\cmark&$0.59$
    \\
     \hline
    \end{tabular}}
}
\parbox{.61\linewidth}{
    \centering
    \small
    \setlength{\tabcolsep}{0.4em}
    \caption{ATSS w. GmaIoU (underlined) makes YOLACT more accurate and  $ \sim 25 \%$ faster mainly owing to less anchors. Thanks to this efficiency, we build maYOLACT-550 with $34.8$ AP and still larger fps than YOLACT. 
    }
    \label{tab:mayolact}
    \resizebox{.57\textwidth}{!}{\begin{tabular}{|c|c|c|c|c|c|} \hline
    \multicolumn{2}{|c|}{Method}&$\mathrm{AP}$&$\mathrm{AP^{box}}$& fps & Anchor \#  \\ \hline
    \multirow{8}{*}{\rotatebox{90}{\footnotesize{GmaYOLACT-550}}}&YOLACT-550&$28.5$&$30.7$&$28$&$\sim 19.2 K$ \\ \cline{2-6}
    &+ \underline{ATSS w. GmaIoU, $P=M$}&\underline{$30.7$}&\underline{$32.8$} &\underline{$\mathbf{32}$}&\underline{$\mathbf{\sim 6.4 K}$}   \\ 
    &+ Carafe FPN \cite{carafe} &$31.4$&$33.4$&$31$&$\mathbf{\sim 6.4 K}$  \\
    &+ DCNv2 \cite{DCNv2} & $33.7$ &$36.4$& $28$ &$\mathbf{\sim 6.4 K}$ \\
    &+ cosine annealing \cite{sgdr} &$34.7$&$37.5$&$28$&$\mathbf{\sim 6.4 K}$\\
    &+ more anchors &$34.7$&$37.9$&$27$&$\sim 12.8 K$\\
    &+ single GPU training &$35.5$&$38.7$&$27$&$\sim 12.8 K$\\
    &+ RS-Loss \cite{RSLoss} &$35.6$&$40.3$&$28$&$\sim 12.8 K$\\    
    \hline
    \end{tabular}}
}
\end{table*}

\subsection{GmaYOLACT Detector: Faster and Stronger}
\label{subsec:GmaYOLACT}
Our ATSS with GmaIoU assigner (underlined in Table \ref{tab:mayolact}) enables us to reduce the number of anchors ($\sim19.2K$ vs $\sim6.4K$) and makes YOLACT $\sim25 \%$ faster ($32$ fps) compared to the baseline YOLACT ($28$ fps)\footnote{Note that YOLACT implemented in mmdetection framework is slower than the official implementation by Bolya et al. \cite{yolact}, who reported $45$fps.}. This highlights the significance of the anchor design strategy for improving the efficiency of real-time instance segmentation models.
Our aim in this section is to improve the standard YOLACT through the integration of our GmaIoU and the recent advances, with the goal of making it competitive with the recent methods while preserving its real-time processing capability\footnote{We use 25fps as the cut-off for ``real-time'' following the common video signal standards (e.g. PAL \cite{pal} and SECAM \cite{secam}) and existing methods \cite{bpd, salientod, basnet, tinieryolo}.}. To achieve that, we  build the GmaYOLACT detector by incorporating the following improvement strategies:
\begin{compactitem}
\item Carafe-FPN \cite{carafe} as the upsampling operation in standard FPN \cite{FeaturePyramidNetwork},
\item Deformable convolutions \cite{DCNv2} in the backbone,
\item Cosine annealing \cite{cosineannealing} with an initial learning rate of $0.008$ which replaces the step learning rate decay,
\item Two anchor base scales of $4$ and $8$ for each pixel and double the number of anchors used by the standard ATSS,
\item Single GPU training following the baseline YOLACT \cite{yolact}, which we find very useful compared to training on multiple GPUs,
\item Rank \& Sort (RS) Loss \cite{RSLoss} as the training objective, resulting in a simple-to-tune model and significant improvement for detection performance.
\end{compactitem}

In our final design, we set the weight of the semantic segmentation head to $0.5$ and employ self-balancing of RS Loss for the other heads, hence we do not tune them. We set the learning rate to $0.008$ and the background removal threshold before  NMS to $0.55$ as we observed that RS Loss yields higher confidence scores compared to the conventional score-based loss functions.

Table \ref{tab:mayolact} demonstrates that without affecting the inference time, that is 28 FPS on a single Nvidia RTX 2080Ti GPU, our GmaYOLACT-550 detector improves the baseline YOLACT-550 significantly (by $7.1$ mask AP and $9.6$ box AP) and reaches $35.6$ mask AP and $40.4$ box AP. We also note that these results also improve upon maYOLACT detector \cite{maIoU}, the earlier version of GmaIoU detector, by around $ \sim 1$ mask AP and $2.5$ box AP (not included in Table \ref{tab:mayolact}).
\begin{table*}[hbt!]
    \centering
    \small
    \setlength{\tabcolsep}{0.2em}
    \caption{Comparison with SOTA on COCO \textit{test-dev}. Our maYOLACT-700 establishes a new SOTA for real-time instance segmentation.
    $^*$ implies our implementation for YOLACT with ATSS w.IoU.
    When a paper does not report a performance measure, N/A is assigned and we reproduce the performance using its repository for completeness (shown by $^\dagger$). }
    \label{tab:testdev}
    \resizebox{\textwidth}{!}{\begin{tabular}{|c|c|c|c|c|c|c|c|c|c|} \hline
        \multicolumn{2}{|c|}{Methods}&Backbone&$\mathrm{AP}$& $\mathrm{AP_{50}}$ & $\mathrm{AP_{75}}$ & $\mathrm{AP_{S}}$ & $\mathrm{AP_{M}}$ & $\mathrm{AP_{L}}$&Reference \\ \hline
    \multirow{6}{*}{\rotatebox{90}{fps $ < 25$}}
    &YOLACT-700 \cite{yolact} & ResNet-101&$31.2$&$50.6$&$32.8$&$12.1$&$33.3$&$47.1$&ICCV 19\\    
    &PolarMask \cite{polarmask}& ResNet-101&$32.1$&$53.7$&$33.1$&$14.7$&$33.8$&$45.3$&CVPR 20\\
    &PolarMask++ \cite{PolarMask-plus}& ResNet-101&$33.8$&$57.5$&$34.6$&$16.6$&$35.8$&$46.2$&TPAMI 21\\
    &RetinaMask \cite{retinamask}& ResNet-101&$34.7$&$55.4$&$36.9$&$14.3$&$36.7$&$50.5$&Preprint\\
    &Mask R-CNN \cite{tensormask}& ResNet-50&$36.8$&$59.2$&$39.3$&$17.1$&$38.7$&$52.1$&ICCV 17\\ 
    &TensorMask \cite{tensormask}& ResNet-101&$37.1$&$59.3$&$39.4$&$17.4$&$39.1$&$51.6$&ICCV 19\\ \hline \hline
    \multirow{15}{*}{\rotatebox{90}{fps $ \geq 25$}}&YOLACT-550 \cite{yolact} & ResNet-50&$28.2$&$46.6$&$29.2$&$9.2$&$29.3$&$44.8$&ICCV 19\\
    &YOLACT-550$^*$& ResNet-50&$29.7$&$49.9$&$30.7$&$11.9$&$32.4$&$42.7$&Baseline\\
    &Solov2-448 \cite{solov2} & ResNet-50 &$34.0$&$54.0$&$36.1$&N/A&N/A&N/A&NeurIPS 20\\
    &Solov2-448$^\dagger$ \cite{solov2-imp} & ResNet-50 &$34.0$&$54.0$&$36.0$&$10.3$&$36.3$&$53.8$&NeurIPS 20\\
    &YOLACT-550++ \cite{yolact-plus} & ResNet-50&$34.1$&$53.3$&$36.2$&$11.7$&$36.1$&$53.6$&TPAMI 20\\
    &YOLACT-550++ \cite{yolact-plus} & ResNet-101&$34.6$&$53.8$&$36.9$&$11.9$&$36.8$&$55.1$&TPAMI 20\\
    &CenterMask-Lite$^\dagger$ \cite{centermask-imp} & VoVNetV2-39 &$35.7$&$56.7$&$37.9$&$18.4$&$37.8$&$47.3$&CVPR 20\\
    &CenterMask-Lite \cite{centermask} & VoVNetV2-39 &$36.3$&N/A&N/A&$15.6$&$38.1$&$53.1$&CVPR 20\\
    &Solov2-512$^\dagger$ \cite{solov2-imp} & ResNet-50 &$36.9$&$57.5$&$39.4$&$12.8$&$39.7$&$\mathbf{57.1}$&NeurIPS 20\\
    &Solov2-512 \cite{solov2} & ResNet-50 &$37.1$&$57.7$&$39.7$&N/A&N/A&N/A&NeurIPS 20\\
    &SparseInst-448 \cite{SprarseInst} & R-50-d-DCN&$35.9$&$56.5$&$37.7$&$12.3$&$37.1$&$57.0$&CVPR 22\\	
    &SparseInst-608 \cite{SprarseInst} & R-50-d-DCN &$37.9$&$59.2$&$40.2$&$15.7$&$39.4$&$56.9$&CVPR 22\\
    \cline{2-10}
    &maYOLACT-550 \cite{maIoU} & ResNet-50&$35.2$ &$56.2$&$37.1$& $14.7$&$38.0$&$51.4$ &BMVC 21\\
    &maYOLACT-700  \cite{maIoU} & ResNet-50&$37.7$ &$59.4$&$39.9$& $18.1$&$40.8$&$52.5$ & BMVC 21\\
    \cline{2-10}
    &GmaYOLACT-550 (\textbf{Ours}) & ResNet-50 &$35.9$ &$57.0$&$37.9$& $14.8$&$38.4$&$53.2$ &\\
    &GmaYOLACT-700 (\textbf{Ours}) & ResNet-50&$\mathbf{38.7}$&$\mathbf{61.0}$&$\mathbf{41.2}$&$\mathbf{19.0}$&$\mathbf{41.9}$&$53.7$&\\
    \hline
    \end{tabular}}
\end{table*}
\subsection{Comparison with State-of-the-art (SOTA)}
\label{subsec:SOTA}
We compare our GmaYOLACT with state-of-the-art methods on COCO \textit{test-dev} split in Table \ref{tab:testdev}.

\noindent \textbf{Comparison with YOLACT variants.} Our GmaYOLACT-550 reaches $35.9$ mask AP and surpasses all YOLACT variants, including those with larger backbones (e.g. YOLACT-550++ with ResNet-101); larger scales (e.g. YOLACT-700); and maYOLACT-550 detector proposed in the initial version of this work \cite{maIoU}. Unlike YOLACT++ \cite{yolact-plus} which is $\sim 25\%$ slower than YOLACT (as seen in Table 6 in Bolya et al. \cite{yolact-plus}), GmaYOLACT-550 has similar inference time with YOLACT-550 and still has $\sim 6$ mask AP increase on COCO \textit{test-dev}, achieving a mask AP of $35.9$.

\noindent \textbf{Comparison with real-time methods.} Without utilizing multi-scale training as in Solov2 \cite{solov2} or a specially designed backbone like in CenterMask \cite{centermask}, our GmaYOLACT-700 still outperforms existing real-time counterparts with a mask AP of $38.7$ at 25fps by at least around $\sim 1$ mask AP. 
Additionally, our top model achieves $61.0$ in terms of the commonly used $\mathrm{AP_{50}}$ metric and has a gap of $\sim 2$ mask $\mathrm{AP_{50}}$ points when compared to the SparseInst as its closest real-time counterpart.

\noindent \textbf{Comparison with other methods.} Our GmaYOLACT is also a strong competitor against relatively slower methods, as shown in Table \ref{tab:testdev}. It outperforms PolarMask++ \cite{PolarMask-plus}, RetinaMask \cite{retinamask}, Mask R-CNN \cite{MaskRCNN}, and TensorMask \cite{tensormask} while being faster. For example, on a RTX 2080Ti GPU, our GmaYOLACT-700 delivers approximately 2x more throughput at 25fps and a nearly 4 mask AP improvement ($37.7$ AP - Table \ref{tab:testdev}) compared to PolarMask++ on ResNet-101 with a test time of 14 fps. It is also about 8x faster than TensorMask using ResNet-101 (i.e. $\sim 3$fps) while maintaining similar performance.
\section{Conclusion}
\label{sec:Conclusion}
We presented GmaIoU, a generalized version of our previous maIoU which provides the flexibility of using either ground-truth mask and its box, or mask alone. For training instance segmentation methods, we used the GmaIoU to designate anchors as positive or negative, utilizing the shape of objects indicated only by the ground-truth segmentation masks. We showed that integrating our GmaIoU with ATSS boosts the runtime performance of the model. With this increased efficiency, we were able to attain SOTA results in real-time.

\section*{Acknowledgments}
This work was partially supported by the Scientific and Technological Research Council of Turkey (T\"UB\.ITAK) through the project titled “Addressing Class Imbalance in Visual Recognition Problems by Measuring Class Imbalance and Using Epistemic Uncertainty (DENGE)” (project no. 120E494).  We also gratefully acknowledge the computational resources kindly provided by TÜBİTAK ULAKBİM High Performance and Grid Computing Center (TRUBA) and METU Robotics and Artificial Intelligence Center (ROMER). Dr. Akbas is supported by the ``Young Scientist Awards Program (BAGEP)'' of Science Academy, Turkey.



\bibliographystyle{elsarticle-num} 
\bibliography{main}





\end{document}